\documentclass{article}

\usepackage[utf8]{inputenc}
\usepackage{latexsym}
\usepackage{url}            %
\usepackage{nicefrac}       %
\usepackage{microtype}      %
\usepackage[scaled]{helvet} %

\usepackage{amsfonts}
\usepackage{amsmath}
\usepackage{amssymb}
\usepackage{mathtools}

\usepackage{graphicx}
\usepackage{adjustbox}
\usepackage{floatrow}
\usepackage{wrapfig}
\graphicspath{{figures/pdf/}}
\DeclareGraphicsExtensions{.pdf}

\usepackage{tikz}
\usetikzlibrary{arrows}
\usetikzlibrary{backgrounds}

\usepackage[font=small,labelfont=bf]{caption}
\usepackage{subcaption}

\usepackage{algorithm}
\usepackage{algorithmic}

\usepackage{booktabs}
\usepackage{colortbl}
\usepackage{multirow}
\usepackage{tabularx}
\usepackage{makecell}

\usepackage[normalem]{ulem}
\useunder{\uline}{\ul}{}

\usepackage{enumitem}
\usepackage{xspace}
\usepackage{todonotes}

\usepackage{macro}

\usepackage[accepted]{hill2019}

\icmltitlerunning{Human Interaction and Interpretability Paper}

\begin{document}

\twocolumn[
\icmltitle{Regularizing Black-box Models for Improved Interpretability}

\icmlsetsymbol{equal}{*}

\begin{icmlauthorlist}
\icmlauthor{Gregory Plumb}{cmu}
\icmlauthor{Maruan Al-Shedivat}{cmu}
\icmlauthor{Eric Xing}{cmu}
\icmlauthor{Ameet Talwalkar}{cmu}
\end{icmlauthorlist}

\icmlaffiliation{cmu}{Machine Learning Department, Carnegie Mellon University, Pittsburgh, USA}

\icmlcorrespondingauthor{Gregory Plumb}{gdplumb@andrew.cmu.edu}

\icmlkeywords{interpretability, transparency}

\vskip 0.3in
]

\printAffiliationsAndNotice{}

\begin{abstract}
Most of the work on interpretable machine learning has focused on designing either inherently interpretable models, which typically trade-off accuracy for interpretability, or post-hoc explanation systems, which lack guarantees about their explanation quality.
We propose an alternative to these approaches by directly regularizing a black-box model for interpretability at training time.
Our approach explicitly connects three key aspects of interpretable machine learning:
(i) the model's innate explainability,
(ii) the explanation system used at test time, and
(iii) the metrics that measure explanation quality.
Our regularization results in substantial improvement in terms of the explanation fidelity and stability metrics across a range of datasets and black-box explanation systems while slightly improving accuracy.
Further, if the resulting model is still not sufficiently interpretable, the weight of the regularization term can be adjusted to achieve the desired trade-off between accuracy and interpretability.  
Finally, we justify theoretically that the benefits of explanation-based regularization generalize to unseen points.
\end{abstract}

\section{Introduction}
\label{sec:introduction}

Complex learning-based systems are increasingly shaping our daily lives, and, in order to monitor and understand these systems, we require clear explanations of model behavior.
While model interpretability has many definitions and is often largely application specific \citep{lipton2016mythos}, local explanations are a popular and powerful tool~\citep{ribeiro2016should}.
Recent work on local interpretability in machine learning ranges from proposals of new models that are interpretable \emph{by-design} \citep[\eg,][]{wang2015falling, caruana2015intelligible} to model-agnostic, \emph{post-hoc} algorithms for interpreting complex, black-box predictors such as ensembles and deep neural networks \citep[\eg,][]{ribeiro2016should,lei2016rationalizing, lundberg2017unified, selvaraju2017grad, kim2018interpretability}.
Despite the variety of technical approaches, the underlying goal of all of these works is to develop an interpretable predictive system that produces two outputs: a prediction and its underlying explanation.

Both interpretability by-design and post-hoc explanation strategies have limitations.
On the one hand, the by-design approaches are restricted to working with model families that provide inherent explainability, potentially at the cost of accuracy.
On the other hand, by performing two disjointed steps, there is no guarantee that post-hoc explainers applied to an arbitrary model will produce explanations of suitable quality.
Moreover, recent approaches that claim to overcome this apparent trade-off between prediction accuracy and explanation quality are in fact by-design proposals that impose certain constraints on the underlying model families they consider~\citep[\eg,][]{alshedivat2017cen, plumb2018model, melis2018towards}. 
In this work, we propose a novel alternative strategy called \emph{Explanation-based Optimization} (\name) that aims to address both of these shortcomings by adding an \emph{interpretability regularizer} to the loss function of an arbitrary predictive model.
A small demo of how our regularizer can influence the explainability and accuracy of a model is in Figure~\ref{fig:demo}.  

\textbf{Illustration.}
To motivate \name, consider a situation where Bob's loan application is denied by a machine learning system (see Figure~\ref{fig:toy} for a toy illustration).
In this setting, a good local explanation can help Bob understand how to improve his application in order to get the loan.
Unfortunately, as we see from Figure~\ref{fig:toy}, a standard model---a multi-layer perceptron trained with SGD---is difficult to explain well because it has many kinks and abrupt changes.
Indeed, we can quantitatively measure the quality of local explanations using the standard \emph{fidelity} \citep{ribeiro2016should, plumb2018model} and \emph{stability} \citep{melis2018towards} explanation metrics. 
To make the learned model more amenable to local explanation, \name augments the objective function with fidelity- or stability-based regularizers, effectively controlling the degree of local explainability.

\begin{figure}[t]
    \centering
    \begin{subfigure}[b]{0.45\textwidth}
    \includegraphics[width=\textwidth]{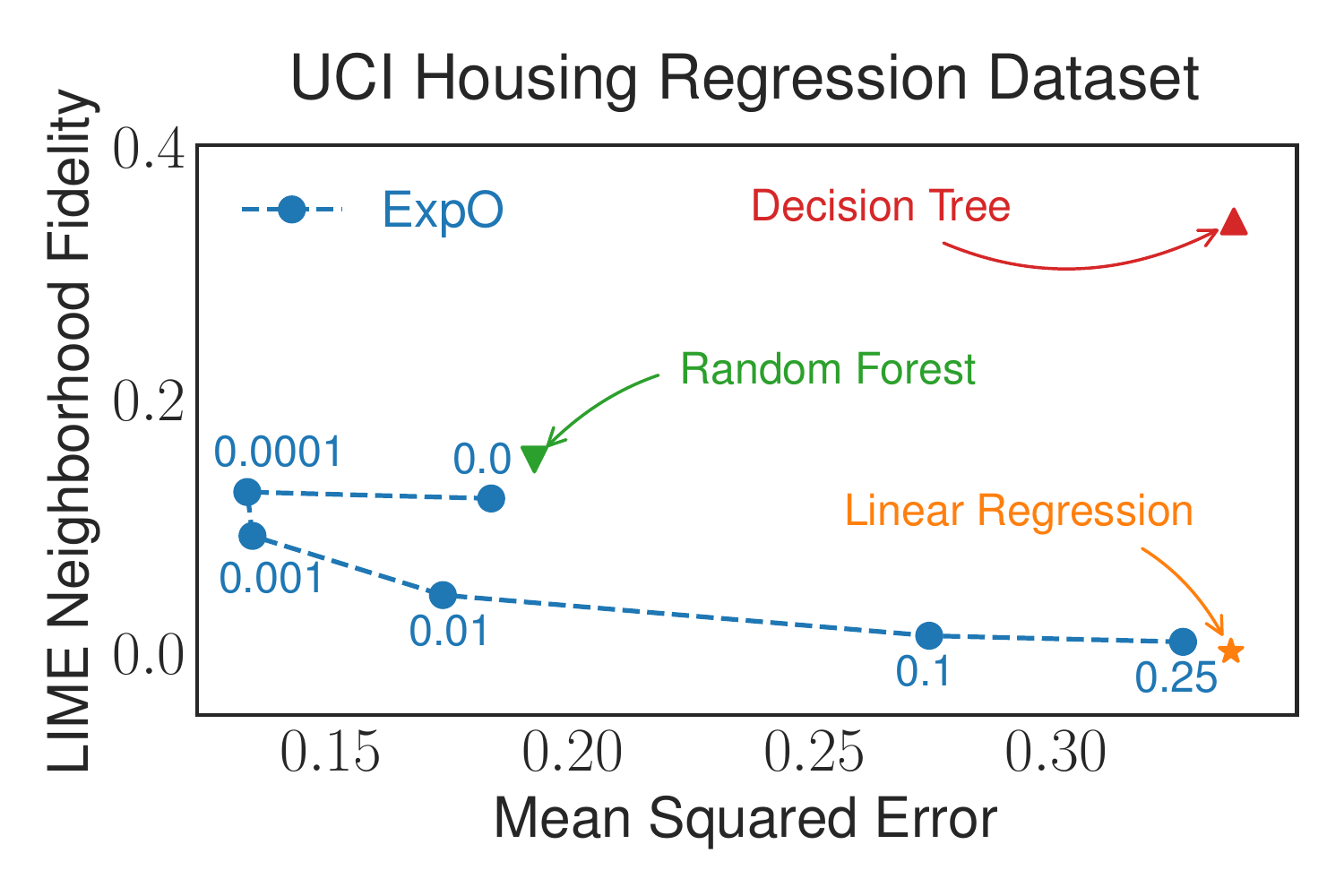}\vspace{-1.5ex}
    \caption{}\label{fig:demo}
    \end{subfigure}
    \begin{subfigure}[b]{0.4\textwidth}
    \includegraphics[width=\textwidth]{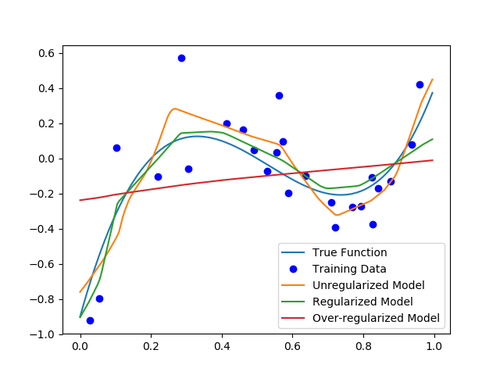}\vspace{-1.5ex}
    \caption{}\label{fig:toy}
    \end{subfigure}
    \caption{\tiny%
    \textbf{(a)} Neighborhood Fidelity of LIME-generated explanations (lower is better) vs. predictive error of several models trained on the UCI Housing regression dataset.
    The values in blue denote regularization weight.  
    \textbf{(b)} The effects of \name on a model predicting hypothetical credit rating.
    The abrupt kinks in the unregularized model make local linear approximations both less faithful to the model and less stable to small perturbations.
    The regularized model is much smoother and therefore easier to explain.}
    \vspace{-1.5ex}
\end{figure}

\vspace{30pt}
The specific contributions of our work are as follows:
\begin{enumerate}[itemsep=1pt,topsep=0pt,leftmargin=14pt]
    \item \textbf{Interpretability Regularizers.}
    We introduce two  explanation regularizers associated with the fidelity and stability explanation metrics.  
    The first, \regf, is designed for semantic features and explainers that directly make predictions, such as \citep{ribeiro2016should, lundberg2017unified,plumb2018model}. 
    The second, \regs, is tailored for non-semantic features (\eg, pixels) and explainers such as saliency maps~\citep{simonyan2013deep}, which identify features that are influential on a prediction.
    Both regularizers are differentiable and can be used to augment the objective function of an arbitrary model.
    In Section~\ref{sec:intuition}, we discuss how they differ from the classical approaches for local approximation and function smoothing.
    \item \textbf{Generalizable Explanation Quality.}
    We analyze the properties of the explanation quality metrics and show that the benefits of our regularization generalize to unseen points.
    Specifically, we derive a bound on the gap between the fidelity of explanations on training and held out points and connect it with the local variance of the learned model.
\end{enumerate}
\textbf{Empirical Results.}
We evaluate models trained with and without the proposed regularizers on a variety of regression and classification tasks with semantic and image features.\footnote{The code for our regularizers and all experiments is at:https://github.com/GDPlumb/ExpO}
We show experimentally that our regularizers slightly improve predictive performance across the nine datasets we consider (seven UCI regression tasks, a medical classification task, and MNIST).
Moreover, from an interpretability perspective, our results demonstrate significant improvement in terms of explanation quality as measured by the fidelity and stability metrics.
In particular, our regularization technique improved explanation fidelity by at least 25\% on the UCI datasets and on the medical classification task; stability on MNIST was improved by orders of magnitude.
\section{Background and Related Work}
\label{sec:background}

In this section, we introduce our notation, provide the necessary background on local explanations, and review the previous work that is most closely related to \name.

Consider a supervised learning problem, where our goal is to estimate a model, $f: \Xc \mapsto \Yc$, $f \in \Fc$, that maps input feature vectors, $x \in \Xc$, to targets, $y \in \Yc$, and is trained using data, $\{x_i, y_i\}_{i=1}^N$.
If the class of functions used for modeling the data is complex, we can understand the behavior of $f$ in some neighborhood, $N_x \in \Pc[\Xc]$ (where $\Pc[\Xc]$ is the space of probability distributions over $\Xc$), by generating a local \emph{explanation}.
We denote algorithms that produce explanations (\ie, \emph{explainers}) as $e: \Xc \times \Fc \mapsto \Ec$, where $\Ec$ is the set of possible explanations.
The choice of $\Ec$ generally depends on whether or not $\Xc$ consists of \emph{semantic features}, and  will be defined more precisely next.

\subsection{Semantic Features}

We call features \emph{semantic} if people can reason about them and understand what it means when their values change (\eg, a person's income, the concentration of a chemical, etc.).
Consequently, local explanations try to predict how the model's output would change if the input was perturbed \cite{ribeiro2016should, lundberg2017unified, plumb2018model}.
Thus, we can define the output space of the explainer as $\Ec_{s} := \{g \in \Gc \mid g: \Xc \mapsto \Yc\}$, where $\Gc$ is a class of interpretable (typically linear) functions.

\paragraph{Fidelity-Metric.}
When the explainer's output space is $\Ec_{s}$, the explanation is defined as a function $g: \Xc \mapsto \Yc$, and it is natural to evaluate how accurately $g$ models $f$ in a neighborhood $N_x$~\citep{ribeiro2016should, plumb2018model}:
\begin{equation}
    \label{eq:fidelity}
    F(f, g, N_x) := \mathbb{E}_{x' \sim N_x}[\left(g(x') - f(x')\right)^2],
\end{equation}
which we refer to as the \emph{neighborhood-fidelity} (NF) metric.
This metric is sometimes evaluated with $N_x$ as a point mass on $x$ and we call this version the \emph{point-fidelity} (PF) metric.
While \citet{plumb2018model} argued that point-fidelity can be misleading because it does not measure generalization of $e(x,f)$ across $N_x$, it has been used for evaluation in the prior work~\citep{ribeiro2016should,lundberg2017unified,ribeiro2018anchors} and we report it in our experiments along with the neighborhood-fidelity for completeness.

\paragraph{Black-box Explanation Systems.} 
Various explainers have been proposed to generate local explanations of the form $g: \Xc \mapsto \Yc$, typically assuming that $g$ is linear.
In particular, LIME~\citep{ribeiro2016should}, one of the most popular black-box explanation systems\footnote{SHAP \citep{lundberg2017unified} is a popular variation of LIME that proposes a theoretically-motivated neighborhood sampling function, but requires explanations to be linear models that act on binary features. 
This requirement is too limiting in our case, hence SHAP is not used in our study.}, solves the following optimization problem:
\begin{equation}
    \label{eq:lime}
    e(x, f) := \argmin_{g \in \Ec_s} F(f, g, N_x) + \Omega(g),
\end{equation}
where $\Omega(e)$ stands for an additive regularizer that encourages certain desirable properties of the explanations (\eg, sparsity). 
LIME's objective function is closely related to the fidelity metric and subsequently to our proposed \regf regularizer.
Consequently, we expect our regularizer to improve the quality of LIME-generated explanations and our experimental results in Section~\ref{sec:exp-fidelity} corroborate this hypothesis.

Along with LIME, we consider another black-box explanation tool, called MAPLE~\citep{plumb2018model}.
It differs substantially from LIME in that its \emph{neighborhood function is learned from the data} rather than specified as a parameter. 
In our experiments, we evaluate the quality of MAPLE-generated local explanations for models regularized via \regf, but do not use MAPLE's learned neighborhood function to define \regf. 
We view this as a good test case to see how optimizing the fidelity metric for one neighborhood generalizes to another one (see Section \ref{sec:methods} for a more detailed discussion of this point).
In Section~\ref{sec:exp-fidelity}, we see that regularizing for LIME neighborhoods improves MAPLE's explanation quality.

\subsection{Non-Semantic Features}
Non-semantic features lack an inherent interpretation, with images being a canonical example.\footnote{%
In general, it is not clear how to interpret perturbations on the pixel level or whether such perturbations result in `real' images.
However, in certain cases such as scientific imaging \cite{de2016jet}, each pixel value may have a precise meaning because of the way the images are processed.}
When $\Xc$ consists of non-semantic inputs, we cannot assign meaning to the difference between $x$ and $x'$, hence it does not make sense to explain the difference between the predictions $f(x)$ and $f(x')$.
As a result, fidelity is not an appropriate explanation metric.
Instead, in this context, local explanations try to identify which parts of the input are particularly influential on a prediction \cite{sundararajan2017axiomatic}.
Consequently, we consider explanations of the form  $\Ec_{ns} := \mathbb{R}^{d}$, where $d$ is the number of features in $\Xc$.

\paragraph{Stability Metric and Saliency Maps.}
When the explainer's output space is $\Ec_{ns}$, the explanation is a vector in $\mathbb{R}^{d}$, and cannot be directly compared to the underlying model itself, as in the case of the fidelity metric. 
Instead, the focus in this setting is on the degree to which the explanation changes between points in a local neighborhood, which we measure using the \emph{stability metric}\cite{melis2018towards}:
\begin{equation} 
    \label{eq:stability}
    \Sc(f, e, N_x) := \mathbb{E}_{x' \sim N_x}[||e(x, f) - e(x', f)||_2^2]
\end{equation}
Various explainers~\citep{sundararajan2017axiomatic, zeiler2014visualizing, shrikumar2016not, smilkov2017smoothgrad, adebayo2018sanity} have been proposed to generate local explanations in $\Ec_{ns}$, with \emph{saliency maps} \citep{simonyan2013deep} being a popular approach that we consider in this work.
Saliency maps assign importance weights to image pixels based on the magnitude of the gradient of the predicted class with respect to the corresponding pixels.

Recent work on model interpretability emphasizes that more stable explanations tend to be more trustworthy~\citep{melis2018towards, ghorbani2017interpretation, alvarez2018robustness}.
Note that the stability metric can also be considered in the context of semantic features in addition to the fidelity metric, and we consider both in our experiments.

\subsection{Related Methods}
A few recently proposed approaches to model interpretability are closely related to our work.
First, self-explaining neural networks (SENN)~\citep{melis2018towards} (a variation of contextual explanation networks~(CEN)~\citep{alshedivat2017cen}) is an interpretable by-design approach that additionally (indirectly) optimizes their models to produce stable explanations.
Second, ``Right For The Right Reasons'' (RTFR)~\citep{ross2017right} selectively penalizes gradients of the output with respect to certain input features at some points to discourage their use by the model.
Finally, a work concurrent with ours~\citep{lee2019functional} proposed to regularize models of structured data to encourage explainability in a way that is similar to \name. 

From a technical standpoint, SENN and RTFR both assume that the local explanation is close to the first order Taylor approximation of the model at that point.
In Section~\ref{sec:intuition}, we demonstrate how Taylor approximations are often quite different from and more difficult to use than the neighborhood-based local explanations that we use in \name.
Further, SENN's regularizer requires the neural network to have a very particular structure and, therefore, unlike \name, cannot by applied to an arbitrary model.
While RTFR's regularization can be used with arbitrary models, it is not directly related to a measure of explanation quality; on the other hand, \name aims to directly improve quality of explanations with respect to a specific metric.

\section{Explanation Optimization}
\label{sec:methods}

\begin{figure}[t]
\begin{minipage}[t]{0.56\textwidth}
\begin{algorithm}[H]
    \caption{\regf Regularizer}
    \label{alg:expo-fidelity-reg}
    \begin{algorithmic}[1]
    \small\vspace{2pt}
    \INPUT $f_\theta$, $x$, $N_x^\mathrm{reg}$, $m$
    \STATE Sample points: \\$x'_1, \dots, x'_m \sim N_x^\mathrm{reg}$
    \STATE Compute predictions: \\
    \begin{equation*}
        \hat y_j(\theta) = f_\theta(x'_j) %
    \end{equation*}\vspace{-2ex}
    \STATE Produce a local linear explanation:
    \begin{equation*}
        {\scriptscriptstyle \beta_x(\theta) = \argmin_\beta \sum\nolimits_{j=1}^m (\hat y_j(\theta) - \beta^\top x'_j)^2}
    \end{equation*}\vspace{-1ex}
    \OUTPUT ${\scriptscriptstyle \frac{1}{m}\sum_{j=1}^m (\hat y_j(\theta) - \beta_x(\theta)^\top x'_j)^2}$
    \end{algorithmic}
\end{algorithm}
\end{minipage}\hspace{2pt}%
\begin{minipage}[t]{0.43\textwidth}
\begin{algorithm}[H]
    \caption{\regs Regularizer}
    \label{alg:expo-stability-reg}
    \begin{algorithmic}[1]
    \small\vspace{2pt}
    \INPUT $f_\theta$, $x$, $N_x^\mathrm{reg}$, $m$
    \STATE Sample points: \\$x'_1, \dots, x'_m \sim N_x^\mathrm{reg}$
    \STATE Compute predictions:
    \begin{equation*}
        \hat y_j(\theta) = f_\theta(x'_j) %
    \end{equation*}
    \OUTPUT ${\scriptstyle \frac{1}{m}\sum_{j=1}^m (\hat y_j(\theta) - f_\theta(x))^2}$
    \end{algorithmic}
\end{algorithm}
\end{minipage}
\end{figure}

Running black-box explainers on arbitrary models does not guarantee the quality of the produced explanations.
To address this, we define a regularizer that can be added to the loss function and used to train an arbitrary model $f$.
Specifically, we want to solve the following optimization problem:
\begin{equation}
    \small
    \label{eq:expo-learning-objective}
    \hat f := \argmin_{f \in \Fc} \frac{1}{N} \sum_{i=1}^N (\Lc(f, x_i, y_i) + \gamma \Rc(f, N_{x_i}^\mathrm{reg}))
\end{equation}
where $\Lc(f, x_i, y_i)$ is a standard predictive loss (\eg, squared error for regression or cross-entropy for classification), $\Rc(f, N_{x_i}^\mathrm{reg})$ is a regularizer that encourages explainability of $f$ in the neighborhood of $x_i$, and $\gamma > 0$ controls the regularization strength.

We define $\Rc(f, N_{x}^\mathrm{reg})$ based on either the neighborhood-fidelity, Eq.~\eqref{eq:fidelity}, or the neighborhood-stability, Eq.~\eqref{eq:stability}.
In order to compute these metrics exactly, we would need to run an explainer algorithm, $e$;  this may be non-differentiable or too computationally expensive to use as a regularizer.  
Thus, for \regf, we approximate $e$ using a local linear model fit on points sampled from $N_x^{reg}$ (Algorithm~\ref{alg:expo-fidelity-reg}).  
For \regs, we simply require that the model's output not change too much across $N_x^{reg}$ (Algorithm~\ref{alg:expo-stability-reg}).\footnote{We note that a similar procedure was explored previously in \citep{zheng2016improving} for adversarial robustness.}

To define a good regularization neighborhood, $N_x^\mathrm{reg}$, requires taking the following into consideration.
On the one hand, we would like $N_x^\mathrm{reg}$ to be similar to $N_x$, as used in Eq.~\ref{eq:fidelity} or Eq.~\ref{eq:stability}, so that the neighborhoods used for regularization and for evaluation match.
On the other hand, we also would like $N_x^\mathrm{reg}$ to be consistent with the `local neighborhood' defined by $e$ internally, which may differ from $N_x$.
For LIME, this is not a problem since the internal definition of the `local neighborhood' is a hyperparameter that we can set.
However for MAPLE, the `local neighborhood' is learned from the data, and hence the regularization and explanation neighborhoods may differ.
Ultimately, we left resolving this tension to future work.  

\textbf{Computational Cost.}
Algorithm \ref{alg:expo-fidelity-reg} could be prohibitively expensive since the number of samples, $m$, from $N_x^{reg}$, has to be proportional to the dimension of $x$, resulting in $O(d^3)$ operations to compute the regularizer for a given point.
In addition to running experiments with \regf, we run experiments with a randomized version of the Algorithm~\ref{alg:expo-fidelity-reg} that randomly selects one dimension of $x$ to perturb according to $N_x^{reg}$ and penalizes the error of a local linear model along that dimension.
This breaks the dependence of the computational cost of the objective function on the dimension of $x$ (bringing it back to $O(1)$) and allows us to compute each gradient step with some constant increase in the number of function evaluations.
We call this variation \regfr.

\subsection{Understanding the Properties of ExpO}
\label{sec:intuition}

\begin{figure}[t]
    \centering
    \includegraphics[width=0.4\textwidth]{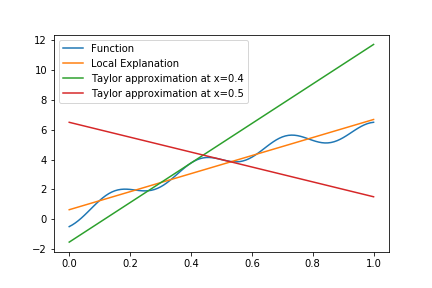}
    \includegraphics[width=0.37\textwidth]{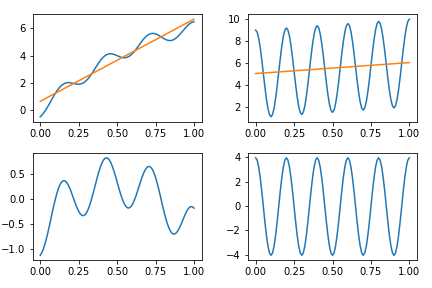}
    \vspace{-1ex}
    \caption{\tiny
    \textbf{Left:} A function (blue), its first order Taylor approximations at $x = 0.4$ (green) and $x = 0.5$ (red), and a local explanation of the function (orange) computed with $x = 0.5$ and $N_x = [0,1]$.
    \textbf{Right (top row):}  Two functions (blue) and their local linear explanations (orange).
    The local explanations were computed with $x = 0.5$ and $N_x = [0,1]$.
    \textbf{Right (bottom row):} The unexplained portion of the function (residuals).
    }
    \label{fig:theory-taylor}
    \label{fig:theory-variance}
\end{figure}

The goal of this section is to compare the behavior of local linear explanations and our regularizer to some existing theoretical function approximations and measures of variance to help develop an intuitive understanding of \name.  
First, we compare neighborhood-based local linear explanations to first order Taylor approximations to show that they can have fundamentally very different behaviors.  
Second, we compare \regf to the Lipchitz Constant (LC) and Total Variation (TV) of the learned function.

\vspace{10pt}
\textbf{Local Explanation vs. Taylor Approximations.}
A natural question to ask is, \emph{Why should we sample from $N_x$ in order to locally approximate $f$ when there are easier and theoretically motivated approximations?}
One possible way to do this is via the Taylor approximation~\citep{melis2018towards}.
The downside of a Taylor approximation-based approach is that such approximation cannot readily be adjusted to different neighborhood scales and its fidelity and stability strictly depend on the learned function.
This can be seen in Figure~\ref{fig:theory-taylor} where the Taylor approximations at two nearby points are both radically different and not faithful to the model outside of an infinitesimal neighborhood.

\textbf{Fidelity-Regularization and the Model's LC or TV.}
From a theoretical perspective, our regularizer is similar to controlling the Lipschitz Constant or Total Variation of $f$ across $N_x$ after removing the part of $f$ explained by $e(x,f)$.  
From an interpretability perspective, there is nothing inherently wrong with having a large LC or TV, which is demonstrated in Figure~\ref{fig:theory-variance}.  
However, once we take into account what can be explained by $e(x,f)$, then upper bounding any one of \regf, the LC, or the TV will upper bound the remaining ones.

\subsection{Generalization of Local Linear Explanations}

To conclude our analysis, we study the quality of local linear explanations in terms of generalization.
Note that \name regularization encourages learning models that are explainable in the neighborhoods of each \emph{training point}.
However, how would this property generalize to unseen points?

We answer this question by providing a generalization bound in terms of neighborhood-fidelity metric.
First, we assume that local linear explanations, $\beta_x$, are obtained by solving the ordinary least squares regression problem (as given in Algorithm~\ref{alg:expo-fidelity-reg}).
The fidelity of the explanation in expectation over the neighborhood $N_x$ can be computed analytically:
\begin{equation}\tiny
    \label{eq:expected-residual-ordinary-regression}
    \begin{split}
        \MoveEqLeft r(f, x) =  \ep[N_x]{f(x')^2} -  \ep[N_x]{f(x') x'}^\top \ep[N_x]{[x' x'^\top]}^{-1} \ep[N_x]{f(x') x'} %
    \end{split}
\end{equation}
where expectation $\ep[N_x]{\cdot}$ is taken with respect to $x'$ over the neighborhood $N_x$.
Note the equality in \eqref{eq:expected-residual-ordinary-regression} is the expected value of the squared residual between $f(x)$ and the optimal local linear explanation, which is upper-bounded by the variance of the model in the corresponding neighborhood.  

For the explanations to generalize, we would like to make sure that the gap between the average fidelity on the training set and the expected fidelity is small with high probability.
More formally, the following inequality should hold:
\begin{equation}
    \label{eq:fidelity-bound-in-probability-generic}
    \prob{\ep{r(f, x)} - \frac{1}{n} \sum_{i=1}^n r(f, x_i) > \varepsilon} < \delta_n(\varepsilon)
\end{equation}
Under certain mild assumptions on the local behavior of $f(x)$, the following proposition specifies a particular bound.
Further, we show empirically that the benefits of our novel regularizers on explanation quality provably generalize to unseen test points.

\begin{proposition}
    Let the neighborhood sampling function $N_x$ be characterized by some parameter $\sigma$ (\eg, the effective radius of a neighborhood) and the variance of the trained model $f(x)$ across all such neighborhoods be bounded by some constant $C(\sigma) > 0$.
    Then, the following bound holds with at least $1 - \delta$ probability:
    \begin{equation}
        \label{eq:fidelity-bound-in-probability-hoeffding}
        \ep{r(f, x)} \leq \frac{1}{n} \sum_{i=1}^n r(f, x_i) + \sqrt{\frac{C^2(\sigma)\log\frac{1}{\delta}}{2n}}
    \end{equation}
\end{proposition}
\begin{proof}
    (Sketch)
    By assumption, the variance of the model $f(x)$ is bounded in each local neighborhood specified by $N_x$.
    Then \eqref{eq:expected-residual-ordinary-regression} implies that each residual is bounded as $0 \leq r(f, x) \leq C(\sigma)$.
    The result then follows by applying Hoeffding's inequality and rearranging the terms.  
\end{proof}

\section{Experimental Results}
\label{sec:experiments}

\begin{wraptable}[10]{r}{0.25\textwidth}
\RawFloats
\centering
\caption{Statistics of the datasets.}
\label{tab:datasets}
\tiny
\resizebox{\textwidth}{!}{
\begin{tabular}{@{}lrr@{}}
\toprule
\textbf{Dataset}    & \textbf{\# samples}   & \textbf{\# dims}          \\ \midrule
autompgs            & 392                   & 8                         \\
communities         & 1993                  & 103                       \\
day                 & 731                   & 15                        \\
housing             & 506                   & 12                        \\
music               & 1059                  & 70                        \\
winequality-red     & 1599                  & 12                        \\
YearPredictionMSD   & 515345                & 90                        \\
SUPPORT2            & 9104                  & 51                        \\ 
MNIST               & 60000                 & 784                       \\ \bottomrule
\end{tabular}}
\end{wraptable}

In our first set of experiments, we demonstrate the effectiveness of \regf and \regfr on datasets with semantic features using several regression problems from the UCI collection \citep{dua2017uci} as well as an in-hospital mortality classification problem.\footnote{\url{http://biostat.mc.vanderbilt.edu/wiki/Main/SupportDesc}.}
Our second experiment demonstrates the effectiveness of \regs for creating saliency maps \citep{simonyan2013deep} on MNIST \citep{mnist}.
Dataset statistics are given in Table~\ref{tab:datasets}.

\subsection{Neighborhood-Fidelity Regularization}
\label{sec:exp-fidelity}

First, we compare models trained without our regularizers to models trained with them.
We report accuracy and three interpretability metrics: (1) Point-Fidelity (PF), (2) Neighborhood-Fidelity (NF), (3) Stability (S) for explanations generated by LIME and MAPLE.
For example, the ``MAPLE-PF'' label corresponds to the Point-Fidelity Metric for explanations produced by MAPLE.

\textbf{Experimental Setup.}
The network architectures and hyper-parameters were chosen by a simple grid search.
All inputs were standardized to have mean zero and variance one (including the response variable for regression problems).  
For the final set of experiments, we set $N_x$ to be $\mathcal{N}(x, \sigma)$ with $\sigma = 0.1$.
Analysis of the effects of different neighborhood sizes is given in Figure~\ref{fig:scale} and shows that the size is not critical (the value of LIME-NF increase only slightly with $\sigma$).
For the UCI regression datasets and the in-hostpital mortality classification task, we set $N_x^{reg}$ to be $\mathcal{N}(x, \sigma)$ with $\sigma = 0.5$ as we found this to produce slightly more accurate and more interpretable models (Figure~\ref{fig:scale}).

\begin{figure}
    \centering
    \begin{subfigure}[b]{0.73\textwidth}
    \includegraphics[width=\textwidth]{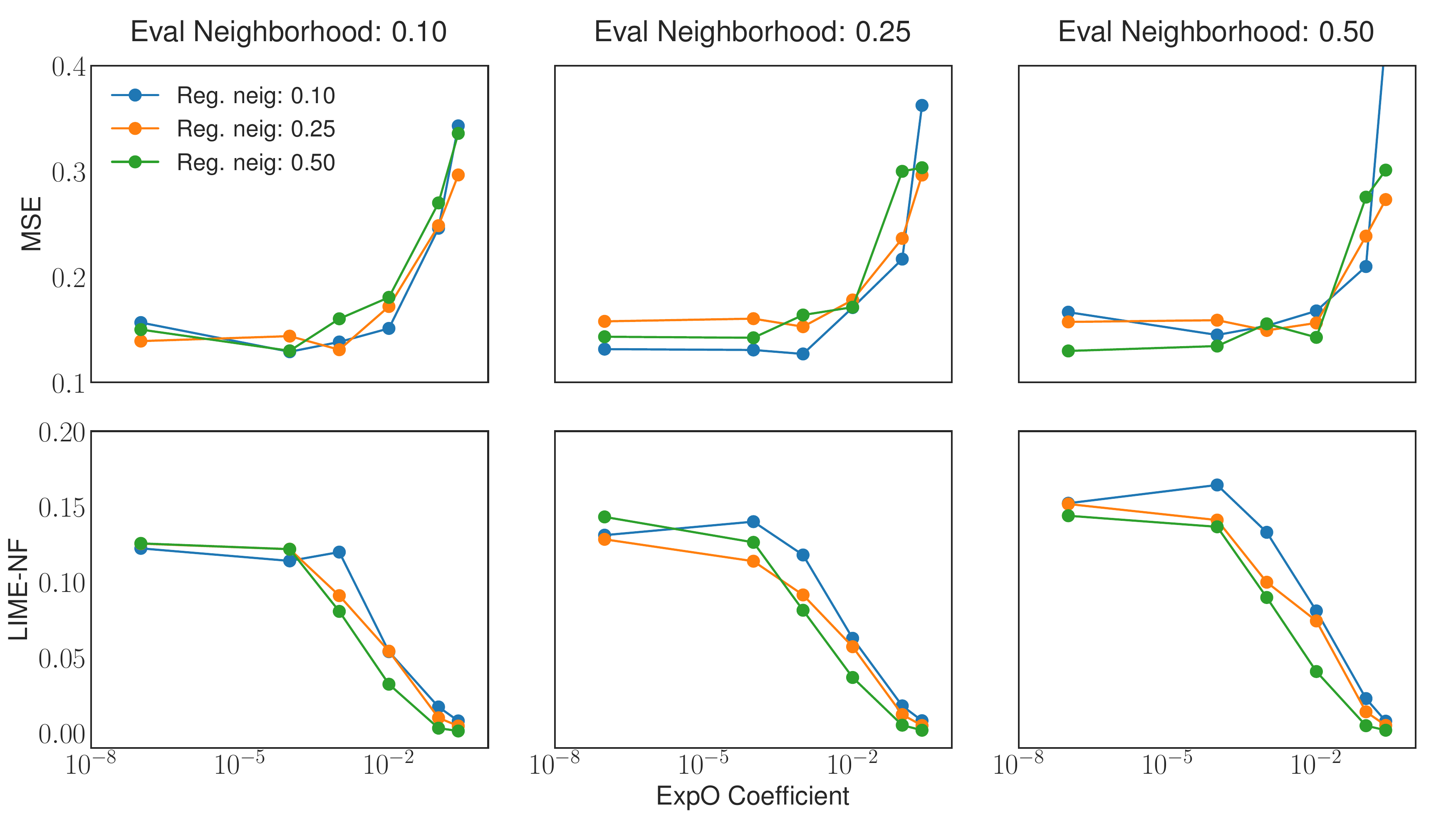}\vspace{-1.6ex}
    \caption{}\label{fig:scale}
    \end{subfigure}
    \quad
    \begin{subfigure}[b]{0.21\textwidth}
    \includegraphics[width=\textwidth]{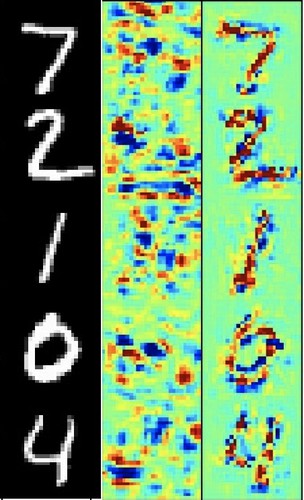}\vspace{2ex}
    \caption{}\label{fig:mnist}
    \end{subfigure}
    \caption{\tiny%
    \textbf{(a)} A comparison showing the effects of the $\sigma$ parameter of $N_x$ and $N_x^{reg}$ on the UCI Housing dataset.  
    The LIME-NF metric grows slowly with $\sigma$ for $N_x$ as expected.  
    Despite being very large, using $\sigma = 0.5$ for $N_x^{reg}$ is generally best for the LIME-NF metric and possibly for accuracy.
    \textbf{(b)} Original MNIST images (left) and saliency maps of an unregularized (middle) and regularized (right) models.}
    \vspace{-2ex}
\end{figure}

\textbf{UCI Regression Experiments.}
The effects of \regf and \regfr on model accuracy and interpretability metrics are in Table \ref{table:uci}.
\regf frequently improved the interpretability metrics by over 50\%, with the smallest improvements being around 25\%.
In fact, our regularization lowered the prediction error on the `communities', `day', and `YearPredictionMSD' datasets, which lets us conclude that it has a small positive effect on accuracy as well as a substantial benefit to the interpretability metrics.
\regfr, while generally having a similar effect, consistently improve interpretability of the models.

We also run experiments on the `YearPredictionMSD' dataset\footnote{The task is to predict release year of song from a set of acoustic features, treated as a regression problem as in ~\citet{bloniarz2016supervised}; the dataset is denoted MSD in Table~\ref{table:uci}.} to understand the scalability of \name to larger tasks.
However, MAPLE-based evaluation was fairly slow on this dataset, and hence we only evaluate the interpretability metrics with respect to LIME on the first 1000 testing points.
Both \regf and \regfr improved LIME's interpretability metrics at least 50\% and both improved the model accuracy.

\textbf{Medical Classification Experiments.}
The SUPPORT2 dataset is used for in-hospital mortality prediction.
The output layer of our models is the softmax over logits for two classes.
Consequently, we run each explanation system on each of the individual logits. 
Table~\ref{table:med} presents the results.
Again, we observe that \regf did not affect the accuracy but did improve the interpretability metrics by 50\% or more.
\regfr slightly decreased accuracy and did not improve the interpretability metrics by as much as \regf, but it did improve them by at least 25\%. 

\begin{table}[t]
\RawFloats
\caption{\tiny%
Uregularized model vs. the same model trained with \regf or \regfr on the UCI regression datasets.
Results are shown across 20 trials (with the standard error in parenthesis).
Statistically significant improvement ($p = 0.05$) due to \regft is denoted in bold and due to \regfrt is underlined.}
\label{table:uci}
\resizebox{\textwidth}{!}{
\begin{tabular}{ll|rrrrrrr}
\topline\headcol
{\bf Metric}    & {\bf Regularizer} & {\bf autompgs}        & {\bf communities}     & {\bf day$^\dagger$} $(10^{-3})$   & {\bf housing}         & {\bf music}           & {\bf winequality.red}     & {\bf MSD}                 \\ \midline
                & None              & 0.14 (0.03)           & {\ul 0.49 (0.05)}     & {\ul 1.000 (0.300)}               & 0.14 (0.05)           & 0.72 (0.09)           & 0.65 (0.06)               & 0.583 (0.018)             \\
{\bf MSE}       & \regft            & 0.13 (0.02)           & {\bf 0.46 (0.03)}     & {\bf 0.002 (0.002)}               & 0.15 (0.05)           & 0.67 (0.09)           & 0.64 (0.06)               & {\bf 0.557 (0.0162)}       \\
                & \regfrt           & 0.13 (0.02)           & 0.55 (0.04)           & 5.800 (8.800)                     & 0.15 (0.07)           & 0.74 (0.07)           & 0.66 (0.06)               & {\ul 0.548 (0.0154)}       \\ \rowmidlinewc\rowcol

                & None              & 0.040 (0.011)         & 0.100 (0.013)         & 1.200 (0.370)                     & 0.14 (0.036)          & 0.110 (0.037)         & 0.0330 (0.0130)           & 0.116 (0.0181)            \\ \rowcol
{\bf LIME-PF}   & \regft            & {\bf 0.011 (0.003)}   & {\bf 0.080 (0.007)}   & {\bf 0.041 (0.007)}               & {\bf 0.057 (0.017)}   & {\bf 0.066 (0.011)}   & {\bf 0.0025 (0.0006)}     & {\bf 0.0293 (0.00709)}     \\ \rowcol
                & \regfrt           & {\ul 0.029 (0.007)}   & {\ul 0.079 (0.026)}   & 0.980 (0.380)                     & {\ul 0.064 (0.017)}   & {\ul 0.080 (0.039)}   & {\ul 0.0029 (0.0011)}     & {\ul 0.057 (0.0079)}     \\ \rowmidlinecw

                & None              & 0.041 (0.012)         & 0.110 (0.012)         & 1.20 (0.36)                       & 0.140 (0.037)         & 0.112 (0.037)         & 0.0330 (0.0140)           & 0.117 (0.0178)            \\ 
{\bf LIME-NF}   & \regft            & {\bf 0.011 (0.003)}   & {\bf 0.079 (0.007)}   & {\bf 0.04 (0.07)}                 & {\bf 0.057 (0.018)}   & {\bf 0.066 (0.011)}   & {\bf 0.0025 (0.0006)}     & {\bf 0.029 (0.007)}     \\ 
                & \regfrt           & {\ul 0.029 (0.007)}   & {\ul 0.080 (0.027)}   & 1.00 (0.39)                       & {\ul 0.064 (0.017)}   & {\ul 0.080 (0.039)}   & {\ul 0.0029 (0.0011)}     & {\ul 0.0575 (0.0079)}     \\ \rowmidlinewc\rowcol

                & None              & 0.0011 (0.0006)       & 0.022 (0.003)         & 0.150 (0.021)                     & 0.0047 (0.0012)       & 0.0110 (0.0046)       & 0.00130 (0.00057)         & 0.0368 (0.00759)          \\ \rowcol
{\bf LIME-S}    & \regft            & {\bf 0.0001 (0.0003)} & {\bf 0.005 (0.001)}   & {\bf 0.004 (0.004)}               & {\bf 0.0012 (0.0002)} & {\bf 0.0023 (0.0004)} & {\bf 0.00007 (0.00002)}   & {\bf 0.00171 (0.00034)}    \\ \rowcol
                & \regfrt           & {\ul 0.0008 (0.0003)} & {\ul 0.018 (0.008)}   & {\ul 0.100 (0.047)}               & {\ul 0.0025 (0.0007)} & 0.0084 (0.0052)       & {\ul 0.00016 (0.00005)}   & {\ul 0.0125 (0.00291)}    \\ \rowmidlinecw

                & None              & 0.0160 (0.0088)       & 0.16 (0.02)           & 1.0000 (0.3000)                   & 0.057 (0.024)         & 0.17 (0.06)           & 0.0130 (0.0078)           & ---                       \\
{\bf MAPLE-PF}  & \regft            & {\bf 0.0014 (0.0006)} & {\bf 0.13 (0.01)}     & {\bf 0.0002 (0.0003)}             & {\bf 0.028 (0.013)}   & 0.14 (0.03)           & {\bf 0.0027 (0.0010)}     & ---                       \\
                & \regfrt           & {\ul 0.0076 (0.0038)} & {\ul 0.092 (0.03)}    & {\ul 0.7600 (0.3000)}             & {\ul 0.027 (0.012)}   & {\ul 0.13 (0.05)}     & {\ul 0.0016 (0.0007)}     & ---                       \\ \rowmidlinewc\rowcol

                & None              & 0.0180 (0.0097)       & 0.31 (0.04)           & 1.2000 (0.3200)                   & 0.066 (0.024)         & 0.18 (0.07)           & 0.0130 (0.0079)           & ---                       \\ \rowcol
{\bf MAPLE-NF}  & \regft            & {\bf 0.0015 (0.0006)} & {\bf 0.24 (0.05)}     & {\bf 0.0003 (0.0004)}             & {\bf 0.033 (0.014)}   & {\bf 0.14 (0.03)}     & {\bf 0.0028 (0.0010)}     & ---                       \\ \rowcol
                & \regfrt           & {\ul 0.0084 (0.0040)} & {\ul 0.16 (0.05)}     & {\ul 0.9400 (0.3600)}             & {\ul 0.032 (0.013)}   & {\ul 0.14 (0.06)}     & {\ul 0.0017 (0.0008)}     & ---                       \\ \rowmidlinecw

                & None              & 0.0150 (0.0099)       & 1.2 (0.2)             & {\ul 0.0003 (0.0008)}             & 0.18 (0.14)           & 0.08 (0.06)           & 0.0043 (0.0020)           & ---                       \\
{\bf MAPLE-S}   & \regft            & {\bf 0.0017 (0.0005)} & {\bf 0.8 (0.4)}       & 0.0004 (0.0004)                   & {\bf 0.10 (0.08)}     & {\bf 0.05 (0.02)}     & {\bf 0.0009 (0.0004)}     & ---                       \\
                & \regfrt           & {\ul 0.0077 (0.0051)} & {\ul 0.6 (0.2)}       & 1.2000 (0.6600)                   & {\ul 0.09 (0.06)}     & {\ul 0.04 (0.02)}     & {\ul 0.0004 (0.0002)}     & ---                       \\ \bottomline
\end{tabular}}\\[-1.5ex]
\flushleft{\scriptsize $^\dagger$The relationship between inputs and targets in UCI Day dataset is very close to linear and hence all errors are orders of magnitude smaller.}
\caption{\tiny%
Uregularized model vs. the same model trained with \regf or \regfr on the SUPPORT2 binary classification dataset.
Each explanation metric computed for both positive and negative class logits.
Results are shown across 10 trials (with the standard error in parenthesis).
Improvement due to \regft and \regfrt over unregularized model is statistically significant ($p = 0.05$) for each metric.}
\label{table:med}
\setlength\tabcolsep{12pt}
\resizebox{\linewidth}{!}{
\begin{tabular}{ll|rrr|rrr}
\topline\headcol
{\bf Output} &  {\bf Regularizer} & {\bf LIME-PF}       & {\bf LIME-NF}         & {\bf LIME-S}          & {\bf MAPLE-PF}        & {\bf MAPLE-NF}        & {\bf MAPLE-S}     \\ \midline
                &  None           & 0.177 (0.063)       & 0.182 (0.065)         & 0.0255 (0.0084)       & 0.024 (0.008)         & 0.035 (0.010)         & 0.34 (0.06)       \\
Positive        &  \regft         & {\bf 0.050 (0.008)} & {\bf 0.051 (0.008)}   & {\bf 0.0047 (0.0008)} & {\bf 0.013 (0.004)}   & {\bf 0.018 (0.005)}   & {\bf 0.13 (0.05)} \\
                &  \regfrt        & {\ul 0.082 (0.025)} & {\ul 0.085 (0.025)}   & {\ul 0.0076 (0.0022)} & {\ul 0.019 (0.005)}   & {\ul 0.025 (0.005)}   & {\ul 0.16 (0.03)} \\ \rowmidlinewc\rowcol

                &  None           & 0.198 (0.078)       & 0.205 (0.080)         & 0.0289 (0.0121)       & 0.028 (0.010)         & 0.040 (0.014)         & 0.37 (0.18)       \\ \rowcol
Negative        &  \regft         & {\bf 0.050 (0.008)} & {\bf 0.051 (0.008)}   & {\bf 0.0047 (0.0008)} & {\bf 0.013 (0.004)}   & {\bf 0.018 (0.005)}   & {\bf 0.13 (0.03)} \\ \rowcol
                &  \regfrt        & {\ul 0.081 (0.026)} & {\ul 0.082 (0.027)}   & {\ul 0.0073 (0.0021)} & {\ul 0.019 (0.006)}   & {\ul 0.024 (0.007)}   & {\ul 0.16 (0.06)} \\ \bottomlinec
\end{tabular}}\\[-2ex]
\flushleft{\,\scriptsize {\bf Accuracy (\%):} None: $83.0 \pm 0.3$, \regft: $83.4 \pm 0.4$, \regfrt:  $82.0 \pm 0.3$.}
\end{table}

\subsection{Stability Regularization}

In this final experiment, we fit a convolutional neural network to MNIST and then evaluate the stability of its saliency maps to perturbations, where $N_x \equiv N_x^{reg} := \mathrm{Unif}(x - 0.05, x + 0.05)$.  
Both an unregularized model and a model trained with \regs achieved the accuracy of 99\%.
This demonstrates one of the practical differences between SENN~\citep{melis2018towards} and our regularizers:
SENN places strict structural constraints on the network and subsequently lowers the testing accuracy to roughly 97\%; this is not the case for \name which can be applied to arbitrary networks.
\name regularization decreased the average $l_2$ distance between the explanation at $x$ and some $x' \sim N_x$ from 6.94 to 0.0008.
Finally, our regularization makes the resulting saliency maps look much better qualitatively by focusing on the presence or absence of certain pen strokes as seen in Figure~\ref{fig:mnist}.

\section{Conclusion}
\label{sec:conclusion}

In this work, we have introduced the novel idea of directly regularizing arbitrary models to be more interpretable. 
We contrasted our regularizers to classical approaches for function approximation and smoothing and provided a generalization bound for them.  
We demonstrated, across a variety of problem settings and explainers, that our regularizers slightly improve model accuracy and improve the interpretability metrics by somewhere from 25\% to orders of magnitude.

\bibliographystyle{unsrtnat}
\bibliography{bibliography}

\end{document}